%% file: samplepaper.tex
\documentclass[runningheads]{llncs}
\usepackage{graphicx}
\graphicspath{{img/}}     

\usepackage{lscape}
\usepackage{hyperref}       
\usepackage{url}            
\usepackage{booktabs}       
\usepackage{amsfonts}       
\usepackage{nicefrac}       
\usepackage{microtype}      
\usepackage{lipsum}
\usepackage{fancyhdr}

\input{cecilia_notations}

\input{columns_and_rows}
\input{mathematical_notations}
\input{ppaurora_notations.tex}

\usepackage[multiple]{footmisc}
\usepackage{graphicx}
\renewcommand{\fw}{ppAURORA}

\newcommand{\cecilia}{CECILIA}
\newcommand{\Sim}{\mathcal{S}}
\newcommand{\Adv}{\mathcal{A}}
\usepackage{natbib}
\usepackage{caption}
\usepackage{subcaption}
\usepackage{amsmath}
\usepackage[linesnumbered,vlined,boxed,commentsnumbered]{algorithm2e}
\usepackage{tabularx}
\usepackage{hhline}
\newcolumntype{R}{>{\centering\columncolor{gray!5!white}\arraybackslash}p{0.5cm}}
\newcolumntype{V}{>{\centering\columncolor{gray!5!white}\arraybackslash}p{1.07cm}}
\newcolumntype{G}{>{\centering\columncolor{gray!10!white}\arraybackslash}m{0.37cm}}
\newcolumntype{M}{>{\centering\columncolor{gray!20!white}\arraybackslash}m{2.45cm}}
\newcolumntype{F}{>{\centering\columncolor{gray!30!white}\arraybackslash}X}

\newcolumntype{S}[2]{>{\centering\columncolor{gray!#1!white}\arraybackslash}m{#2cm}}

\begin{document}
\title{ppAURORA: Privacy Preserving Area Under Receiver Operating Characteristic and Precision-Recall Curves}
\titlerunning{ppAURORA}

\author{Ali Burak Ünal\inst{1,3}\orcidID{0000-0002-7279-620X} \and
Nico Pfeifer \inst{2,3}\orcidID{0000-0002-4647-8566} \and
Mete Akgün \inst{1,3}\orcidID{0000-0003-4088-2784}} 

\authorrunning{}
%
\institute{Medical Data Privacy Preserving Machine Learning (MDPPML), University of Tübingen, Germany \and
Methods in Medical Informatics, University of Tübingen, Germany \and
Institute for Bioinformatics and Medical Informatics (IBMI), University of Tübingen, Germany\\
\email{\{ali-burak.uenal,nico.pfeifer,mete.akguen\}@uni-tuebingen.de}}
\maketitle              
\begin{abstract}
Computing an AUC as a performance measure to compare the quality of different machine learning models is one of the final steps of many research projects. Many of these methods are trained on privacy-sensitive data and there are several different approaches like $\epsilon$-differential privacy, federated machine learning and cryptography if the datasets cannot be shared or used jointly at one place for training and/or testing. In this setting, it can also be a problem to compute the global AUC, since the labels might also contain privacy-sensitive information. There have been approaches based on $\epsilon$-differential privacy to address this problem, but to the best of our knowledge, no exact privacy preserving solution has been introduced. In this paper, we propose an MPC-based solution, called \fw{}, with private merging of individually sorted lists from multiple sources to compute the exact AUC as one could obtain on the pooled original test samples. With \fw{}, the computation of the exact area under precision-recall and receiver operating characteristic curves is possible even when ties between prediction confidence values exist. We use \fw{} to evaluate two different models predicting acute myeloid leukemia therapy response and heart disease, respectively. We also assess its scalability via synthetic data experiments. All these experiments show that we efficiently and privately compute the exact same AUC with both evaluation metrics as one can obtain on the pooled test samples in plaintext according to the semi-honest adversary setting.

\end{abstract}
\section{Introduction}
Recently, privacy preserving machine learning studies aimed at protecting sensitive information during training and/or testing of a model in scenarios where data is distributed between different sources and cannot be shared in plaintext \citep{mohassel2017secureml,wagh2018securenn,juvekar2018gazelle,mohassel2018aby3,unal2022cecilia,damgaard2019new,byali2020flash,PatraS20}. However, privacy protection in the computation of the area under curve (AUC), which is one of the most preferred methods to compare different machine learning models with binary outcome, has not been addressed sufficiently. There are several differential privacy based approaches in the literature for computing the receiver operating characteristic (ROC) curve \citep{chaudhuri2013stability,boyd2015differential,chen2016differentially}. Briefly, they aim to protect the privacy of the data by introducing noise into the computation so that one cannot obtain the original data employed in the computation. However, due to the nature of differential privacy, the resulting AUC is different from the one which could be obtained using non-perturbed prediction confidence values (PCVs) when noise is added to the PCVs \citep{sun2022differentially}. For the precision-recall (PR) curve, there even exists no such studies in the literature. As a general statement, private computation of the exact AUC has never been addressed before to the best of our knowledge. 

In this paper, we propose a 3-party computation based \textbf{p}rivacy \textbf{p}reserving \textbf{a}rea \textbf{u}nder \textbf{r}eceiver \textbf{o}perating characteristic and precision-\textbf{r}ec\textbf{a}ll curves (\fw{}). For this purpose, we utilize \cecilia{} \citep{unal2022cecilia}, which has efficient functionalities to realize several operations in a privacy preserving way. The most important missing operation of it is the division operation. In order to address the necessity of an efficient, private and secure computation of the exact AUC, we adapt the division operation of SecureNN \citep{wagh2018securenn}. Since the building blocks of \cecilia{} require less communication rounds than SecureNN, we implemented the division operation of SecureNN using the building blocks of \cecilia{}. Using \fw{}, we compute the area under the PR curve (AUPR) and ROC curve (AUROC). We address two different cases of ROC curve in \fw{} by two different versions of AUROC computation. The first one is designed for the computation of the exact AUC using PCVs with no tie. In case of a tie of PCVs of samples from different classes, this version just approximates the metric based on the order of the samples, having a problem when values of both axes of ROC curve plot change at the same time. In order to compute the exact AUC even in case of a tie, we introduce the second version of AUROC with a slightly higher communication cost than the first approach. Along with the privacy of the resulting AUC, since the labels are also kept secret during the whole computation, both versions are capable of protecting the information of the number of samples belonging to the classes from all participants of the computation. Otherwise, such information could have been used to obtain the order of the labels of the PCVs \citep{whitehill2019does}. Furthermore, since we do not provide the data sources with the ROC curve, they cannot regenerate the underlying true data. Therefore, both versions are secure against such attacks \citep{matthews2013examination}. We utilized the with-tie version of AUROC computation to compute the AUPR since the values of both axes can change at the same time even if there is no tie. To the best of our knowledge, \fw{} is the first study addressing the privacy preserving computation of AUPR.

\section{Motivation}
\fw{} can enable the privacy preserving evaluation of a model in a collaborative way. Especially when there are parties with insufficient test samples, even if these parties obtain the collaboratively trained model, they cannot reliably evaluate the predictions of this model on their test samples. The result of AUC on such a small set of test samples could vary significantly, which makes the reliability of the model evaluation questionable. Figure \ref{fig:auc_stabilization} demonstrates the stability of AUROC on varying numbers of samples.

\begin{figure}[h!]
    \centering
    \includegraphics[width=0.6\linewidth]{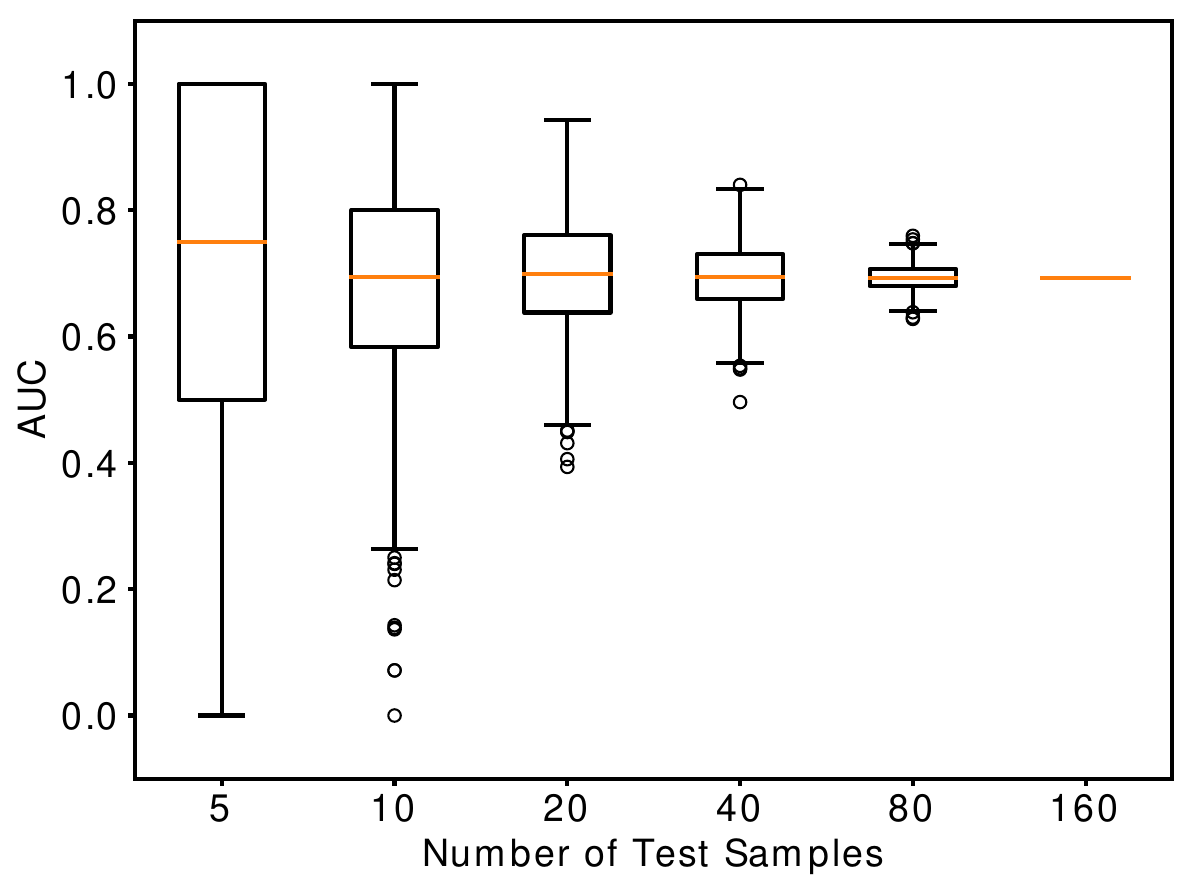}
    \caption{AUROC for varying number of test samples randomly chosen from the whole dataset}
    \label{fig:auc_stabilization}
\end{figure}

To elaborate more on how \fw{} can contribute to the community, let us imagine a scenario where a model is trained collaboratively using an MPC-based framework \citep{mohassel2017secureml,wagh2018securenn,damgaard2019new} or federated machine learning framework \citep{li2020deepfed} or any other method allowing such privacy preserving training. Once the model is obtained, the participating parties can perform predictions on the model using their test samples to evaluate it. However, as it is shown in Figure \ref{fig:auc_stabilization}, the parties with fewer data cannot reliably determine the performance of the model. Instead of individual evaluation of the model, which could lead to incorrect assessment of the model's performance, they can utilize \fw{} to evaluate it collaboratively and obtain the result of this evaluation as if it was performed on the pooled test samples of the parties without sacrificing the privacy of neither the labels nor the predictions of the samples.

\section{Preliminaries}
\textbf{Security Model:} In this study, we aim to protect the privacy of the PCVs and the labels of the samples from parties, the ranking of these samples in the globally sorted list and the resulting AUC. We prove the full security of our solution (i.e., privacy and correctness) in the presence of semi-honest adversaries that follow the protocol specification, but try to learn information from the execution of the protocol. We consider a scenario where a semi-honest adversary corrupts a single server and an arbitrary number of data owners in the simulation paradigm \citep{Lindell17,Canetti01} where two worlds are defined: the real world where parties run the protocol without any trusted party, and the ideal world where parties make the computation through a trusted party. Security is modeled as the view of an adversary called a simulator $\Sim$ in the ideal world, who cannot be distinguished from the view of an adversary $\Adv$ in the real world. The universal composability framework \citep{Canetti01} introduces an adversarial entity called environment $\mathcal{Z}$, which gives inputs to all parties and reads outputs from them. The environment is used in modeling the security of end-to-end protocols where several secure protocols are used arbitrarily. Security here is modeled as \textit{no environment can distinguish if it interacts with the real world and the adversary $\Adv$ or the ideal world and the simulator $\Sim$}. We also provide privacy in the presence of a malicious adversary corrupting any single server, which is formalized by \citet{araki2016high}. The privacy is formalized by saying that a malicious party, which arbitrarily deviates from the protocol description, cannot learn anything about the inputs and outputs of the honest parties. 

\textbf{Notations:} In our secure protocols, we use additive secret sharing over the ring $\mathbb{Z}_{\ring}$ where $\ring = 2^{64}$ to benefit from the natural modulo of CPUs of most modern computers. We denote two shares of $x$ over $\mathbb{Z}_{\ring}$ with ($\langle x \rangle_0$, $\langle x\rangle_1$).


\subsection{Secure Multi-party Computation}
Secure multi-party computation was proposed in the 1980s \citep{Yao:1986:GES:1382439.1382944,Goldreich:1987:PAM:28395.28420}. These studies showed that multiple parties can compute any function on inputs without learning anything about the inputs of the other parties. Let us assume that there are $n$ parties $\party_1,\ldots,\party_n$ and $\party_j$ has a private input $x_j$ for $j \in \{1,\ldots,n\}$. All parties want to compute the arbitrary function $(y_1,\ldots,y_n) = f(x_1,\ldots,x_n)$ and get the corresponding result $y_j$. MPC allows the parties to compute the function through an interactive protocol and $\party_j$ to learn only $y_j$.

In \fw{}, we employ \cecilia{}, which has three computing parties, $\party_0, \party_1$ and $\party_2$, and uses 2-out-of-2 additive secret sharing where an $\ell$-bit value $x$ is shared additively in a ring among $\party_0$ and $\party_1$ as the sum of two values. For $\ell$-bit secret sharing of $x$, we have $\langle x \rangle_{0}  +\langle x \rangle_{1} \equiv x \mod L$ where $\party_i$ knows only $\langle x \rangle_{i} $ and $i \in \{0,1\}$. All arithmetic operations are performed in the ring $\mathbb{Z}_{\ringA}$. 


\subsection{Area Under Curve}
One of the most common ways to summarize the plot-based model evaluation metrics is area under curve (AUC). It calculates the area under the curve of a plot-based model. It is applicable to various different evaluation metrics. Among those, we mention ROC curve and the PR curve, which are the focus of this study.

\subsubsection{Area Under ROC Curve (AUROC)}
In machine learning problems with a binary outcome, the ROC curve is very effective to take the sensitivity and the specificity of the classifier into account by plotting the false positive rate (FPR) on the x-axis and the true positive rate (TPR) on the y-axis. AUC summarizes this plot by measuring the area between the line and the x-axis, which is the area under the ROC curve (AUROC). Let us assume that $\nsamples$ is the number of test samples, $V \in [0,1]^\nsamples$ contains the sorted PCVs of test samples in descending order, $T \in [0,1]^\nsamples$ and $F \in [0,1]^\nsamples$ contain the corresponding TPR and FPR values, respectively, where the threshold for entry $i$ is set to $V[i]$, and $T[0] = F[0] = 0$. In case there is no tie in $V$, the privacy-friendly AUROC computation is as follows:
\begin{equation} \label{eq:auc_wo_tie}
\small
    AUROC = \sum_{i = 1}^{\nsamples} \Big( T[i] \cdot (F[i] - F[i-1]) \Big)
\end{equation}

This formula just approximates the exact AUROC in case of a tie in $V$ depending on the order of the samples. As an extreme example, let $V$ have $10$ samples with the same PCV. Let the first $5$ samples have label $1$ and the second $5$ samples have label $0$. Such a setting outputs $AUROC=1$ with the first version of AUROC. On the contrary, if we have samples with $0$ at the beginning and samples with $1$ later, we obtain $AUROC=0$ with the no-tie version of AUROC. In order to define an accurate formula for the AUROC in case of such a tie condition, let $\xi$ be the vector of indices in ascending order where the PCV of the sample at that index and the preceding one are different for $0 < |\xi| \leq \nsamples$ where $|\xi|$ denotes the size of the vector. Assuming that $\xi[0] = 0$, the computation of AUROC in case of a tie can be done as follows:
\begin{equation} \label{eq:auc_with_tie}
\small
    \begin{aligned}
        AUROC = \sum_{i = 1}^{|\xi|} \Big( & T[\xi[i-1]] \cdot (F[\xi[i]] - F[\xi[i-1]]) + \\ & \dfrac{(T[\xi[i]] - T[\xi[i-1]]) \cdot (F[\xi[i]] - F[\xi[i-1]])}{2} \Big)
    \end{aligned}
\end{equation}
As Equation \ref{eq:auc_with_tie} indicates, one only needs TPR and FPR values on the points where the PCV changes to obtain the exact AUROC. We will benefit from this observation in the privacy preserving AUROC computation when there is a tie condition in the PCVs.

\subsubsection{Area Under PR Curve (AUPR)}
The PR curve evaluates the models with a binary outcome by plotting recall on the x-axis and precision on the y-axis, and it is generally preferred over AUROC for scenarios with class imbalances. The AUC summarizes this plot by measuring the area under the PR curve (AUPR). Since both precision and recall can change at the same time even without a tie, we measure the area by using the Equation \ref{eq:auc_with_tie} where $T$ becomes the precision and $F$ becomes the recall.

\section{\fw{}}

In this section, we give the description of our protocol for \fw{}. In \fw{}, we have data owners that outsource their PCVs and the ground truth labels in secret shared form and three non-colluding servers that perform 3-party computation on secret shared PCVs to compute the AUC. The data sources start the protocol by outsourcing the labels and the predictions of their test samples to the servers. Afterward, the servers perform the desired calculation privately. Finally, they send the shares of the result back to the data sources. The communication between all parties is performed over a secure channel (e.g., TLS).

\textbf{Outsourcing:} At the start of \fw{}, each data owner $H_i$ has a list of PCVs and corresponding ground truth labels for $i \in \{1,\ldots,n\}$. Then, each data owner $H_i$ sorts its whole list $\pl_i$ according to PCVs in descending order, divides it into two additive shares $\pl_{i_0}$ and $\pl_{i_1}$, and sends them to $\party_0$ and $\party_1$, respectively. We refer to $\party_0$ and $\party_1$ as \textit{proxies}.

\textbf{Sorting:} After the outsourcing phase, $\party_0$ and $\party_1$ obtain the shares of individually sorted lists of PCVs of the data owners. Afterwards, the proxies need to perform a merging operation on each pair of individually sorted lists and continue with the merged lists until they obtain the global sorted list of PCVs. This can be considered as the leaves of a binary tree merging into the root node, which is, in our case, the global sorted list. Due to the high complexity of privacy preserving sorting, we decided to make the sorting parametric to adjust the trade-off between privacy and practicality. Let $\delta = 2a + 1$ be this parameter that determines the number of PCVs that will be added to the global sorted list in each iteration for $a \in \mathbb{N}$, and let $\pl_{i_k}$ and $\pl_{j_k}$ be the shares of two individually sorted lists of PCVs in $\party_k$s for $k \in \{0,1\}$ and $|\pl_{i}| \geq |\pl_{j}|$ where $|.|$ is the size operator. At the beginning, the proxies privately compare the lists elementwise. They utilize the results of the comparison in $\mathsf{MUX}$s to privately exchange the shares of PCVs in each pair, if the PCV in $\pl_j$ is larger than the PCV in $\pl_i$. In the first $\mathsf{MUX}$, they input the share in $\pl_{i_k}$ to $\mathsf{MUX}$ first and then the share in $\pl_{j_k}$ along with the share of the result of the comparison to select the larger of the PCVs. They move the results of the $\mathsf{MUX}$ to $\pl_{i_k}$. In the second $\mathsf{MUX}$, they reverse the order to select the smaller of the PCVs and move it to $\pl_{j_k}$. We call this stage \textit{shuffling}. Then, they move the top PCV of $\pl_{i_k}$ to the merged list of PCVs. If $\delta \neq 1$, then they continue comparing the top PCVs in the lists and moving the largest of them to the merged list. Once they move $\delta$ PCVs to the merged list, they shuffle the lists again, and if $|\pl_{j_k}| > |\pl_{i_k}|$, then they switch the list $\pl_{j_k}$ and $\pl_{i_k}$. Until finishing up the PCVs in $\pl_{i_k}$, the proxies follow shuffling-moving cycle.

The purpose of the shuffling is to increase the number of candidates for a specific position and, naturally, lower the chance of matching a PCV in the individually sorted lists to a PCV in the merged list. The highest possible chance of matching is $50\%$. This results in a very low chance of guessing the matching of whole PCVs in the list. Regarding the effect of $\delta$ on the privacy, it is important to note that $\delta$ needs to be an odd number to make sure that shuffling always leads to an increment in the number of candidates. An even value of $\delta$ may cause ineffective shuffling during the sorting. Furthermore, $\delta = 1$ provides the utmost privacy, which means that the chance of guessing the matching of the whole PCVs is 1 over the number of all possible merging of those two individually sorted lists. However, the execution time of sorting with $\delta = 1$ can be relatively high. For $\delta \neq 1$, the execution time can be low but the number of possible matching of PCVs in the individually sorted list to the merged list decreases in parallel to the increment of $\delta$. As a guideline on the choice of $\delta$, one can decide it based on how much privacy loss any matching could cause on the specific task. In case of $\delta \neq 1$ and $|\pl_{j_k}| = 1$ at some point in the sorting, the sorting continues as if it had just started with $\delta = 1$ to make sure that the worst case scenario for guessing the matching can be secured. More details of the sorting phase are in the Appendix.

\textbf{Division ($\mathsf{DIV}$):} Considering the necessity of a division operation for the exact AUC and that the utilized framework \cecilia{} does not have a division operation, we adapted the division operation from SecureNN \citep{wagh2018securenn}. However, instead of the building blocks of SecureNN, we use the building blocks of \cecilia{} to implement the division operation, since they have less communication round complexities. The basic idea is to employ long division in order to find the quotient. Even though the division operation of SecureNN is rather a normalization operation, which requires the denominator to be larger than the nominator, it is still useful for the exact AUC computation. In AUROC, TP of TPR and FP of FPR, which is also recall of AUPR, are always smaller than or equal to the number of positive and negative labels, respectively. Similarly, the precision of AUPR also satisfies this requirement. Therefore, we use this division operation in \fw{} to compute the exact AUC.

\subsection{Secure Computation of AUROC} 
Once $\party_0$ and $\party_1$ obtain the global sorted list of PCVs, they calculate the AUROC based on this list by employing one of the versions of AUROC depending on whether there exists a tie in the list.

\subsubsection{Secure AUROC Computation without Ties}
\label{sec:aucwotie}
In Algorithm \ref{alg:auc1}, we compute the AUROC as shown in Equation \ref{eq:auc_wo_tie} by assuming that there is no tie in the sorted list of PCVs. At the end of the secure computation, the shares of numerator $N$ and denominator $D$ are computed. Since $N$ is always greater than or equal to $D$, we can utilize the division of SecureNN to obtain $AUROC=N/D$. With the help of high numeric value precision of the results, most of the machine learning algorithms yield different PCVs for samples. Therefore, this version of computing the AUROC is applicable to most machine learning tasks. However, in case of a tie between samples from two classes in the PCVs, it does not guarantee the exact AUROC. Depending on the order of the samples, it approximates the score. To have a more accurate AUROC, we propose another version of AUROC computation with a slightly higher communication cost in the next section.

\IncMargin{1.5em}
\begin{algorithm}[!ht]
\footnotesize
\DontPrintSemicolon
\SetKwInOut{Input}{input}\SetKwInOut{Output}{output}
\Input{$\langle \pl \rangle_i = (\{\langle con_1\rangle_i,\langle label_1\rangle_i\},...,$ $\{\langle con_\nsamples \rangle_i, \langle label_\nsamples \rangle_i\} )$, $\langle \pl \rangle_i$ is a share of the global sorted list of PCVs, and labels}
For each $i \in \{0, 1\}$, $\party_i$ executes Steps $2$-$11$\;
$\langle TP\rangle_i \gets 0$, $\langle P\rangle_i \gets 0$, $\langle pFP\rangle_i \gets 0$, $\langle N\rangle_i  \gets 0$\;
\ForEach{item $\langle t\rangle \in \langle \pl \rangle$}{%
      $\langle TP\rangle_i \gets \langle TP\rangle_i + \langle t.label\rangle_i$\;
      $\langle P\rangle_i \gets \langle P\rangle_i + i$\;
      $\langle FP\rangle_i \gets \langle P\rangle_i - \langle TP\rangle_i$\;
      $\langle A\rangle_i \gets \mathsf{MUL}(\langle TP\rangle_i,\langle FP\rangle_i-\langle pFP\rangle_i)$\;
      $\langle N\rangle_i \gets \langle N\rangle_i+\langle A\rangle_i$\;
      $\langle pFP\rangle_i \gets \langle FP\rangle_i$\;
}
$\langle D\rangle_i \gets \mathsf{MUL}(\langle TP\rangle_i,\langle FP\rangle_i)$\;
$\langle ROC\rangle_i \gets \mathsf{DIV}(\langle N\rangle_i,\langle D\rangle_i)$\;
\captionsetup{width=\linewidth}
\caption{Secure AUROC computation without ties}
\label{alg:auc1}
\end{algorithm}
\DecMargin{1.5em}

\IncMargin{1.5em}
\begin{algorithm}[!ht]
\footnotesize
\DontPrintSemicolon
\SetKwInOut{Input}{input}\SetKwInOut{Output}{output}
\Input{$\langle \pl \rangle_i = (\{\langle con_1\rangle_i,\langle label_1\rangle_i\},$ $,...,$ $\{\langle con_\nsamples \rangle_i,\langle label_\nsamples \rangle_i\} )$, $\langle \pl \rangle_i$ is a share of the global sorted list of PCVs, and labels}
For each $i \in \{0, 1\}$, $\party_i$ executes Steps $2$-$14$\;

$\langle TP\rangle_i \gets 0$, $\langle P\rangle_i \gets 0$, $\langle pFP\rangle_i \gets 0$, $\langle pTP\rangle_i \gets 0$, $\langle N_1\rangle_i \gets 0$, $\langle N_2\rangle_i \gets 0$\;

\ForEach{item $\langle t\rangle_i \in \langle \pl \rangle_i$}{%
      $\langle TP\rangle_i \gets \langle TP\rangle_i + \langle t.label\rangle_i$\;
      
      $\langle P\rangle_i \gets \langle P\rangle_i + i$\;
      
      $\langle FP\rangle_i \gets \langle P\rangle_i - \langle TP\rangle_i$\;
      
      $\langle A \rangle_i \gets \mathsf{MUL}([\langle pTP\rangle_i, \langle TP\rangle_i - \langle pTP\rangle_i], [\langle FP\rangle_i-\langle pFP\rangle_i, \langle FP\rangle_i-\langle pFP\rangle_i])$\;
      
      $\langle A\rangle_i \gets \mathsf{MUL}(\langle A\rangle_i, [\langle t.con\rangle_i, \langle t.con\rangle_i])$\;
      
      $\langle N_1\rangle_i \gets \langle N_1\rangle_i+\langle A[0] \rangle_i$\;
      
      $\langle N_2\rangle_i \gets \langle N_2\rangle_i+\langle A[1] \rangle_i$\;
      
      [$\langle pre\_FP \rangle_i, \langle pre\_TP\rangle_i] \gets \mathsf{MUX}([\langle pFP\rangle_i, \langle pTP\rangle_i], [\langle FP\rangle_i, \langle TP\rangle_i], \newline [\langle t.con\rangle_i, \langle t.con\rangle_i])$\;
}
$\langle N\rangle_i \gets 2\cdot\langle N_1\rangle_i+\langle N_2\rangle_i$\;
$\langle D\rangle_i \gets 2\cdot\mathsf{MUL}(\langle TP\rangle_i,\langle FP\rangle_i)$\;
$\langle ROC\rangle_i \gets \mathsf{DIV}(\langle N\rangle_i,\langle D\rangle_i)$\;
\caption{Secure AUROC computation with tie}
\label{alg:auc2}
\end{algorithm}
\DecMargin{1.5em}

\IncMargin{1.5em}
\begin{algorithm}[!t]
\footnotesize
\DontPrintSemicolon
\SetKwInOut{Input}{input}\SetKwInOut{Output}{output}
\Input{$\langle C\rangle_i = (\langle con_1\rangle_i,...,$ $\langle con_\nsamples \rangle_i)$, $\langle C\rangle_i$ is a share of the global sorted list of PCVs, $\nsamples$ is the number of PCVs}
$\party_0$ and $\party_1$ hold a common random permutation $\pi$ for $\nsamples$ items\;
$\party_0$ and $\party_1$ hold a list of common random values $R$\;
$\party_0$ and $\party_1$ hold a list of common random permutation $\sigma$ for $\ell$ items\;
For each $i \in \{0, 1\}$, $\party_i$ executes Steps $5$-$13$\;
\For{$j\leftarrow 1$ \KwTo $\nsamples - 1$}{
      $\langle C[j]\rangle_{i} \gets (\langle C[j]\rangle_{i} - \langle C[j+1]\rangle_{i})$\;
      \uIf{$i = 0$}{
        $\langle C[j]\rangle_{i} = \ring - \langle C[j]\rangle_{i}$\;
      }
      $\langle C[j]\rangle_{i} = \langle C[j]\rangle_{i} \oplus R[j]$\;
      $\langle C[j]\rangle_{i} = \sigma_j(\langle C[j]\rangle_{i})$\;
}
$\langle D\rangle_{i}=\pi(\langle C\rangle_{i})$\;
Insert arbitrary number of dummy zero and non-zero values to randomly chosen locations in $\langle D\rangle_{i}$\;
$\party_i$ sends $\langle D\rangle_{i}$ to $P_{2}$\;
$\party_2$ reconstructs $D$ by computing $\langle D\rangle_{0}\oplus\langle D\rangle_{1}$\;
\ForEach{item $\langle d\rangle \in \langle D\rangle$}{%
    \If{$d>0$}{
        $d \gets 1$\;
    }
}
$\party_2$ creates new shares of $D$, denoted by $\langle D\rangle_0$ and $\langle D\rangle_1$, and sends them to $\party_0$ and $\party_1$, respectively.\;
For each $i \in \{0, 1\}$, $\party_i$ executes Steps $18$-$21$\;
Remove dummy zero and non-zero values from $\langle D\rangle_i$\;
$\langle C\rangle_i=\pi^{-1}(\langle D\rangle_i)$\;
\For{$j\leftarrow 1$ \KwTo $\nsamples - 1$}{
    $\langle \pl[j].con\rangle_i \gets \langle C[j]\rangle_i$\;
}
$\langle \pl[\nsamples].con\rangle_i \gets i$\;
\caption{Secure detection of ties}
\label{alg:ties}
\end{algorithm}
\DecMargin{1.5em}

\subsubsection{Secure AUROC Computation with Ties}
\label{sec:auctie}
To detect ties in the list of PCVs, $\party_0$ and $\party_1$ compute the difference between each PCV and its following PCV. $\party_0$ computes the modular additive inverse of its shares. The proxies apply a common random permutation to the bits of each share in the list to prevent $\party_2$ from learning the non-zero relative differences. They also permute the list of shares using a common random permutation to shuffle the order of the real test samples. Then, they send the list of shares to $\party_2$. $\party_2$ XORes two shares and maps the result to one, if it is greater than zero and zero otherwise. Then, proxies privately map PCVs to zero if they equal to their previous PCV and one otherwise. This phase is depicted in Algorithm \ref{alg:ties}. In Algorithm \ref{alg:auc2}, $\party_0$ and $\party_1$ use these mappings to take only the PCVs which are different from their subsequent PCV into account in the computation of the AUROC based on Equation \ref{eq:auc_with_tie}. In Algorithm \ref{alg:auc2}, $\mathsf{DIV}$ that we adapted from SecureNN can be used since the numerator is always smaller than or equal to the denominator, as in the AUROC computation described in Section \ref{sec:aucwotie}.


\subsection{Secure AUPR Computation}
As in the AUROC computation described in Section \ref{sec:auctie}, $\party_0$ and $\party_1$ map a PCV in the global sorted list to zero if it equals the previous PCV and one otherwise by running Algorithm \ref{alg:ties}. Then, we use Equation \ref{eq:auc_with_tie} to calculate AUPR as shown in Algorithm \ref{alg:prc}. The most significant difference of AUPR from AUROC with tie computation is that the denominator of each precision value is different in the AUPR calculation. Thus, we need to perform computation of precision for each iteration in advance, which requires a vectorized division operation before iterating the list of PCVs mapped to one.

\IncMargin{1.5em}
\begin{algorithm}[!b]
\footnotesize
\DontPrintSemicolon
\SetKwInOut{Input}{input}\SetKwInOut{Output}{output}
\Input{$\langle \pl \rangle_i = (\{\langle con_1\rangle_i,\langle label_1\rangle_i\},...,$ $\{\langle con_\nsamples \rangle_i,\langle label_\nsamples \rangle_i\} )$, $\langle \pl \rangle_i$ is a share of the global sorted list of PCVs, and labels}
$\party_0$ and $\party_1$ hold a common random permutation $\pi$ for $\nsamples$ items\;

For each $i \in \{0, 1\}$, $\party_i$ executes Steps $3$-$19$\;

$\langle TP[0]\rangle_i \gets 0$, $\langle RC[0]\rangle_i \gets 0$, $\langle pPC\rangle_i \gets i$, $\langle pRC\rangle_i \gets 0$, $\langle N_1\rangle_i \gets 0$, $\langle N_2\rangle_i \gets 0$\;

\For{$j \leftarrow 1$ \KwTo $\nsamples$}{
      $\langle TP[j]\rangle_i \gets \langle TP[j-1]\rangle_i + \langle \pl[j].label\rangle_i$\;
      
      $\langle RC[j]\rangle_i \gets \langle RC[j-1]\rangle_i + i$\;
}
$\langle T\_TP\rangle_{i}=\pi(\langle TP\rangle_{i})$\;
$\langle T\_RC\rangle_{i}=\pi(\langle RC\rangle_{i})$\;
$\langle T\_PC \rangle_i \gets \mathsf{DIV}(\langle T\_TP \rangle_i, \langle T\_RC \rangle_i)$\;
$\langle PC\rangle_{i}=\pi'(\langle T\_PC\rangle_{i})$\;    
\For{$j \leftarrow 1$ \KwTo $\nsamples$}{
      $\langle A\rangle_i \gets \mathsf{MUL}([\langle pPC\rangle_i \langle RC[j]\rangle_i - \langle pRC\rangle_i], [\langle RC[j]\rangle_i - \langle pRC\rangle_i, \langle PC[j]\rangle_i-\langle pPC\rangle_i])$\;
      
      $\langle A\rangle_i \gets \mathsf{MUL}(\langle A\rangle_i, [\langle \pl[j].con\rangle_i, \langle \pl[j].con\rangle_i])$\;
      
      $\langle N_1\rangle_i \gets \langle N_1\rangle_i + \langle A[0] \rangle_i$\;
      
      $\langle N_2\rangle_i \gets \langle N_2\rangle_i+\langle A[1] \rangle_i$\;

      $[\langle pPC \rangle_i, \langle pRC \rangle_i] \gets \newline \mathsf{MUX}( [\langle pPC \rangle_i, \langle pRC \rangle_i], [\langle PC[j] \rangle_i, \langle RC[j] \rangle_i], \newline [\langle \pl[j].con \rangle_i, \langle \pl[j].con \rangle_i])$ \;    
}

$\langle N\rangle_i \gets 2\cdot\langle N_1\rangle_i+\langle N_2\rangle_i$\;
$\langle D\rangle_i \gets 2\cdot \langle TP[\nsamples]\rangle_i$\;
$\langle PRC\rangle_i \gets \mathsf{DIV}(\langle N\rangle_i,\langle D\rangle_i)$\;
\caption{Secure AUPR computation}
\label{alg:prc}
\end{algorithm}
\DecMargin{1.5em}

\section{Security Analysis}
In this section, we provide semi-honest simulation-based security proofs for the computations of \fw{} based on the security of the utilized building blocks of \cecilia{}.

\begin{lemma}
\label{lemma:auc}
The protocol in Algorithm \ref{alg:auc1} securely computes AUROC in the $(\mathcal{F}_{\mathsf{MUL}},\mathcal{F}_{\mathsf{DIV}})$ hybrid model.  
\end{lemma}
\begin{proof}
In the protocol, we separately calculate the numerator $N$ and the denominator $D$ of the AUROC, which can be expressed as $AUROC=\frac{N}{D}$. Let us first focus on the computation of $D$. It is equal to the multiplication of the number of samples with label $1$ by the number of samples with label $0$. In the end, we have the number of samples with label $1$ in $TP$ and calculate the number of samples with label $0$ by $P - TP$. Then, the computation of $D$ is simply the multiplication of these two values. In order to compute $N$, we employed Equation \ref{eq:auc_wo_tie}. We have already shown the denominator part of it. For the numerator part, we need to multiply the current $TP$ by the change in $FP$ and sum up these multiplication results. $\langle A\rangle \gets \mathsf{MUL}(\langle TP\rangle,\langle FP\rangle-\langle pFP\rangle)$ computes the contribution of the current sample on the denominator and we accumulate all the contributions in $N$, which is the numerator part of Equation \ref{eq:auc_wo_tie}. Therefore, we can conclude that we correctly compute the AUROC.

Next, we prove the security of our protocol. $\party_i$ where $i \in \{0,1\}$ sees $\{\langle A \rangle\}_{j \in \nsamples}$, $\langle D \rangle$ and $\langle ROC \rangle$, which are fresh shares of these values. Thus the view of $\party_i$ is perfectly simulatable with uniformly random values. 
\end{proof}

\begin{lemma}
\label{lemma:aucdetec}
The protocol in Algorithm \ref{alg:ties} securely marks the location of ties in the list of prediction confidences.
\end{lemma}

\begin{proof}
For the correctness of our protocol, we need to prove that for each index $j$ in $\pl$, $L[j].con=0$ if $(C[j]-C[j+1])=0$, $L[j].con=1$, otherwise. We first calculate the difference of successive items in $C$. Assume we have two additive shares $(\langle a \rangle_0,\langle a \rangle_1)$ of $a$ over the ring $\mathbb{Z}_\ring$. If $a=0$, then $(\ring - \langle a \rangle_0) \oplus \langle a \rangle_1 = 0$ and if $a\neq0$, then $(\ring - \langle a\rangle_0) \oplus \langle a\rangle_1 \neq 0$ where $\ring - \langle a\rangle_0$ is the additive modular inverse of $\langle a\rangle_0$. We use this fact in our protocol. $\party_0$ computes the additive inverse of each item $\langle c\rangle_0$ in $\langle C\rangle_0$ which is denoted by $\langle c\rangle_0'$, XORes $\langle c\rangle_0'$ with a common random number in $R$, which is denoted by $\langle c\rangle_0''$ and permutes the bits of $\langle c\rangle_0''$ with a common permutation $\sigma$ which is denoted by $\langle c\rangle_0'''$. $\party_1$ XORes each item $\langle c\rangle_1$ in $\langle C\rangle_1$ with a common random number in $R$ which is denoted by $\langle c\rangle_1''$ and permutes the bits of $\langle c\rangle_1''$ with a common permutation $\sigma$ which is denoted by $\langle c\rangle_1'''$. $\party_i$ where $i \in \{0,1\}$ permutes values in $\langle C\rangle_i'''$ by a common random permutation $\pi$ which is denoted by $\langle D\rangle_i$. After receiving $\langle D\rangle_0$ and $\langle D\rangle_1$, $\party_2$ maps each item $d$ of $D$ to $0$ if $\langle d\rangle_0^\prime \oplus \langle d\rangle_1 = 0$ which means $\langle d\rangle_0 + \langle d\rangle_1 = 0$ and maps $1$  if $\langle d\rangle_0^\prime \oplus \langle d\rangle_1 \neq 0$ which means $\langle d\rangle_0 + \langle d\rangle_1 \neq 0$. After receiving a new share of $D$ from $\party_2$, $\party_i$ where $i \in \{0,1\}$ removes dummy values and permutes remaining values by $\pi^{-1}$. Therefore, our protocol correctly maps items of $C$ to $0$ or $1$.

We next prove the security of our protocol. $\party_i$ where $i \in \{0,1\}$ calculates the difference of successive prediction values. The view of $\party_2$ is $D$, which includes real and dummy zero values. $\party_i$ XORes each item of $\langle C\rangle_i$ with fresh boolean shares of zero, applies a random permutation to bits of each item of $\langle C\rangle_i$, applies a random permutation $\pi$ to $\langle C\rangle_i$ and add dummy zero and non-zero values. Thus the differences, the positions of the differences, and the distribution of the differences are completely random. The number of zero and non-zero values are not known to $\party_2$ due to dummy values. With common random permutations $\sigma_{j \in \nsamples}$ and common random values $R[j], j \in \nsamples$, each item in $C$ is hidden. Thus $\party_2$ can not infer anything about real values in $C$. Furthermore, the number of repeating predictions is not known to $\party_2$ due to the random permutation $\pi$.
\end{proof}

\begin{lemma}
\label{lemma:auctie}
The protocol in Algorithm \ref{alg:auc2} securely computes AUROC in ($\mathcal{F}_{\mathsf{MUL}}$,$\mathcal{F}_{\mathsf{MUX}}$,$\mathcal{F}_{\mathsf{DIV}}$) hybrid model.  
\end{lemma}
\begin{proof}
In order to compute the AUROC in case of a tie, we utilize Equation \ref{eq:auc_with_tie}, of which we calculate the numerator and the denominator separately. The calculation of the denominator $D$ is the same as Lemma \ref{lemma:auc}. The computation of the numerator $N$ has two different components, which are $N_1$ and $N_2$. $N_1$, more precisely the numerator of $T[i-1] * (F[i] - F[i-1])$, is similar to the no-tie version of privacy preserving AUROC computation. This part corresponds to the rectangle areas in the ROC curve.  The decision of adding this area $A$ to the cumulative area $N_1$ is made based on the result of the multiplication of $A$ by $L.con$. $L.con=1$ indicates if the sample is one of the points of prediction confidence change, $0$ otherwise. If it is $0$, then $A$ becomes $0$ and there is no contribution to $N_1$. If it is $1$, then we add $A$ to $N_1$. On the other hand, $N_2$, which is the numerator of $(T[i] - T[i-1]) * (F[i] - F[i-1])$, accumulates the triangular areas. We compute the possible contribution of the current sample to $N_2$. In case this sample is not one of the points that the prediction confidence changes, which is determined by $L.con$, then the value of $A$ is set to $0$. If it is, then $A$ remains the same. Finally, $A$ is added to $N_2$. Since there is a division by $2$ in the second part of Equation \ref{eq:auc_with_tie}, we multiply $N_1$ by $2$ to make them have common denominator. Afterwards, we sum $N_1$ and $N_2$ to obtain $N$. In order to have the term $2$ in the common denominator, we multiply $D$ by $2$. As a result, we correctly compute the denominator and the nominator of the AUROC. 

Next, we prove the security of our protocol. $\party_i$ where $i \in \{0,1\}$ sees $\{\langle A \rangle\}_{j \in \nsamples}$, $\{\langle pFP \rangle\}_{j \in \nsamples}$, $\{\langle pTP \rangle\}_{j \in \nsamples}$, $\langle D \rangle$ and $\langle ROC \rangle$, which are fresh shares of these values. Thus the view of $\party_i$ is perfectly simulatable with uniformly random values. 
\end{proof}

\begin{lemma}
\label{lemma:prc}
The protocol in Algorithm \ref{alg:prc} securely computes AUPR in ($\mathcal{F}_{\mathsf{MUL}}$,$\mathcal{F}_{\mathsf{MUX}}$,$\mathcal{F}_{\mathsf{DIV}}$) hybrid model.  
\end{lemma}

\begin{proof}
In order to compute the AUPR, we utilize Equation \ref{eq:auc_with_tie} of which we calculate the numerator and the denominator separately. We nearly perform the same computation with the AUROC computation in case of a tie. The main difference is that we need to perform a division to calculate each precision value because denominators of each precision value are different. The rest of the computation is the same with the computation in Algorithm \ref{alg:auc2}. The readers can follow the proof of Lemma \ref{lemma:auctie}. 

Next, we prove the security of our protocol. $\party_i$ where $i \in \{0,1\}$ sees $\{\langle T\_PC \rangle\}_{j \in \nsamples}$, $\{\langle A \rangle\}_{j \in \nsamples}$, $\{\langle pPC \rangle\}_{j \in \nsamples}$, $\{\langle pRC \rangle\}_{j \in \nsamples}$ and $\langle PRC \rangle$, which are fresh shares of these values. Thus the view of $\party_i$ is perfectly simulatable with uniformly random values. 
\end{proof}

\begin{lemma}
\label{lemma:sorting}
The sorting protocol in Section 5 securely merges two sorted lists in  ($\mathcal{F}_{\mathsf{CMP}}$,$\mathcal{F}_{\mathsf{MUX}}$) hybrid model.  
\end{lemma}

\begin{proof}
First, we prove the correctness of our merge sort algorithm. $\pl_1$ and $\pl_2$ are two sorted lists. In the merging of $\pl_1$ and $\pl_2$, the corresponding values are first compared using the secure $\mathsf{CMP}$ operation. The larger values are placed in $\pl_1$ and the smaller values are placed in $\pl_2$, after the secure  $\mathsf{MUX}$ operation is called twice. This process is called \textit{shuffling} because it shuffles the corresponding values in the two lists. After the shuffling process, we know that the largest element of the two lists is the top element of $\pl_1$. Therefore, it is removed and added to the global sorted list $\pl_3$. On the next step, the top elements of $\pl_1$ and $\pl_2$ are compared with the  $\mathsf{CMP}$ method. The comparison result is reconstructed by $\party_0$ and $\party_1$ and the top element of $\pl_1$ or $\pl_2$ is removed based on the result of $\mathsf{CMP}$ and added to $\pl_3$. The selection operation also gives the largest element of $\pl_1$ and $\pl_2$ because $\pl_1$ and $\pl_2$ are sorted. We show that shuffling and selection operations give the largest element of two sorted lists. This ensures that our merge sort algorithm that only uses these operations correctly merges two sorted lists in ordered manner.

Next, we prove the security of our merge sort algorithm. In the shuffling operation, $\mathsf{CMP}$ and $\mathsf{MUX}$ operations are called. 
$\mathsf{CMP}$ outputs fresh shares of comparison of corresponding values in $\pl_1$ and $\pl_2$. Shares of these comparison results are used in $\mathsf{MUX}$ operations and $\mathsf{MUX}$ operation generates fresh shares of the corresponding values. Therefore, $\party_0$ and $\party_1$ cannot precisely map these values to the values in $\pl_1$ and $\pl_2$. In the selection operation, $\mathsf{CMP}$ is called and the selection is performed based on the reconstructed output of $\mathsf{CMP}$.  $\party_0$ and $\party_1$ are still unable to map the values added to $L_3$ to the values in $\pl_1$ and $\pl_2$ precisely because at least one shuffling operation took place before these repeated selection operations. Shuffling and $\delta-1$ selection operations are performed repeatedly until the $\pl_1$ is empty. After each shuffling operation, the fresh share of the larger corresponding values in $\pl_1$ and the fresh share of the smaller corresponding values in $\pl_2$ are stored. The view of $\party_0$ and $\party_1$ are perfectly simulatable with random values due to the shuffling process performed at regular intervals. 

It is possible in some cases to use unshuffled values in selection operations. To prevent this, the following rules are followed in the execution of the merge protocol. If there are two lists that do not have the same length, the longer list is chosen as $\pl_1$. If the $\delta$ is greater than the length of the $\pl_2$ list, it is set to the largest odd value smaller or equal to the length of $\pl_2$ so that the unshuffled values that $\pl_1$ may have are not used in selection processes. If the length of $\pl_2$ is reduced to $1$ at some point in the sorting, the $\delta$ is set to $1$. Thus $\pl_2$ will have $1$ element until the end of the merge and shuffling is done before each selection. After moving $\delta$ values to the sorted list, if the length of $\pl_2$ is greater than the length of $\pl_1$, we switch the list.
\end{proof}


\subsection{Privacy against Malicious Adversaries}
\citet{araki2016high} defined the notion of privacy against malicious adversaries in the client-server setting. In this setting, the servers performing secure computation on the shares of the inputs to produce the shares of the outputs do not see the plain inputs and outputs of the clients. This notion of privacy says that a malicious party cannot break the privacy of input and output of the honest parties. This setting is very similar to our setting. In our framework, two parties exchange a seed which is used to generate common random values between them. Two parties randomize their shares using these random values, which are not known to the third party. It is very easy to add fresh shares of zero to outputs of two parties with common random values shared between them. In our algorithms, we do not state the randomization of outputs with fresh shares of zero. 
Thus, our framework provides privacy against a malicious party by relying on the security of a seed shared between two honest parties.

\begin{figure*}[!ht]
    \centering
    \captionsetup[subfigure]{justification=centering}
    \begin{subfigure}{0.47\linewidth}
        \includegraphics[width=0.95\linewidth]{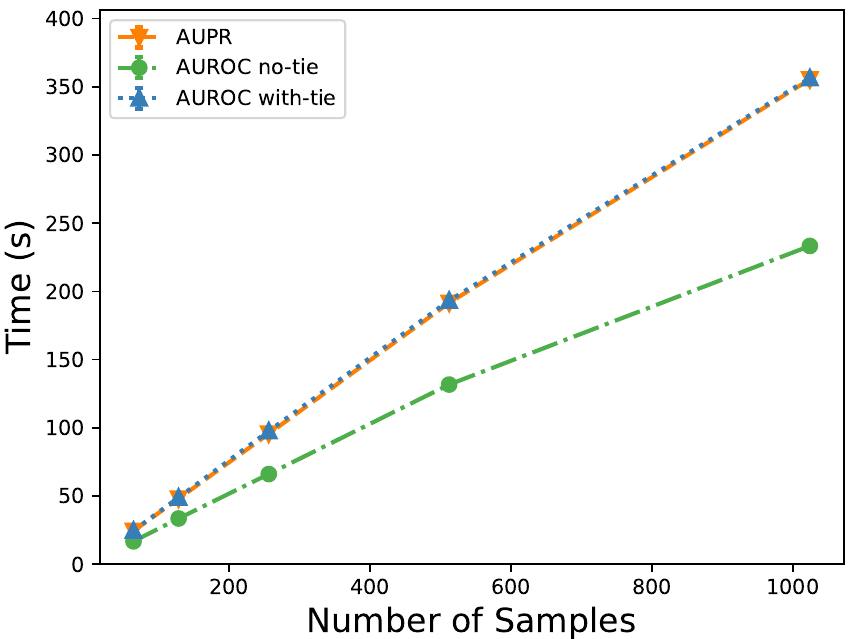}
        \caption{}
        \label{fig:n_samples}
    \end{subfigure}
    \begin{subfigure}{0.47\linewidth}
        \centering
        \includegraphics[width=0.95\linewidth]{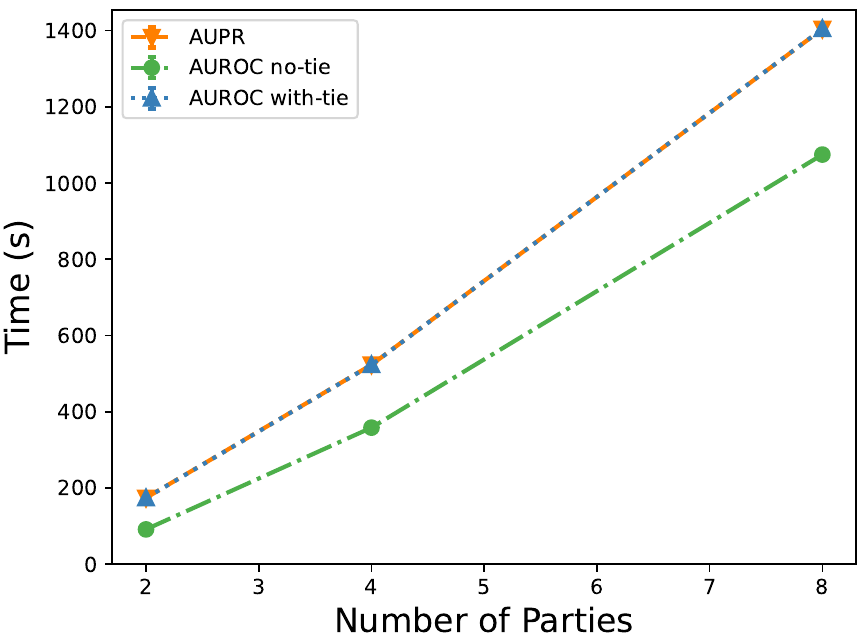}
        \caption{}
        \label{fig:n_parties}
    \end{subfigure}
    \begin{subfigure}{0.48\linewidth}
        \centering
        \includegraphics[width=0.95\linewidth]{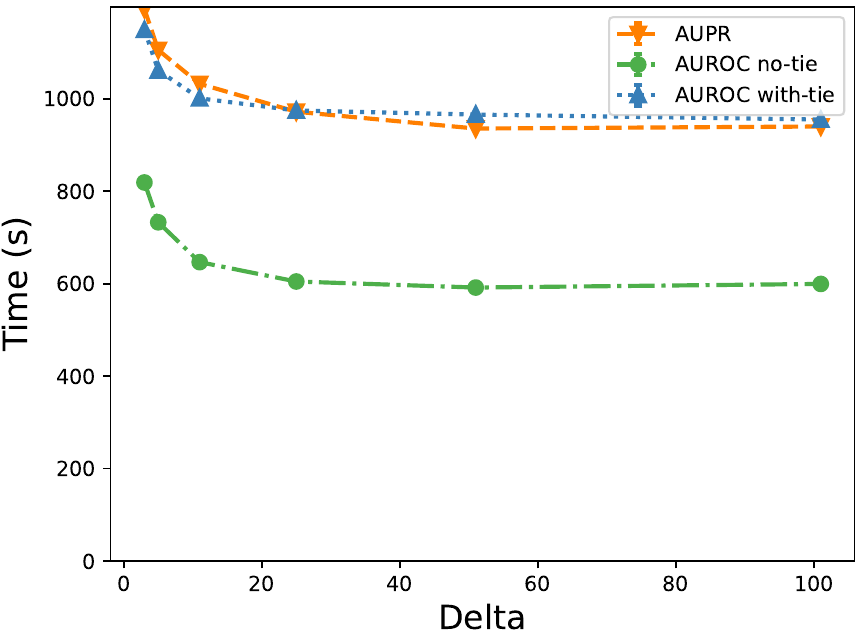}
        \caption{}
        \label{fig:delta}
    \end{subfigure}
    \caption{\textbf{(a)} The execution time of various settings to evaluate the scalability of \fw{} to the number of samples for a fixed number of parties and \textbf{(b)}  to the number of parties for a fixed number of samples in each party. \textbf{(c)} The effect of $\delta$ on the execution time is shown.}
    \label{fig:exe_time}
\end{figure*}

\section{Dataset}
To demonstrate the correctness of \fw{} and its applicability to a real-life problem, we utilized the Acute Myeloid Leukemia (AML) dataset \footnote{\tiny{\url{https://www.synapse.org/\#!Synapse:syn2700200}}}$^,$\footnote{\tiny{\url{https://www.synapse.org/\#!Synapse:syn2501858}}} from the first subchallenge of the DREAM Challenge \citep{noren2016crowdsourcing}. We chose the submission of the team with the lowest score in the leaderboard with accessible files, which is the team \textit{Snail}. The training dataset has $191$ samples, among which $136$ patients have complete remission. Additionally, we used the UCI Heart Disease dataset \footnote{\tiny{\url{https://archive.ics.uci.edu/ml/datasets/heart+disease}}} for the correctness analysis of \fw{}. In the test set, we have $54$ samples with binary outcome.

Moreoever, we aimed to analyze the scalability of \fw{} for different settings. For this purpose, we generated a synthetic dataset with no restriction other than having the PCVs between $0$ and $1$.


\section{Results} \label{sec:results}
In this section, we will give the details of our experiments conducted via \fw{} to demonstrate its correctness and scalability.

\subsection{Experimental Setup} 
We conducted our experiments on both LAN and WAN settings. In the LAN setting, we ran the experiments with $0.18$ ms round trip time (RTT). In the WAN setting, we simulated the network connection with $10$ ms RTT.


\subsection{Correctness Analysis} 
We conducted the experiments for the correctness analysis on the LAN setting. To assess the correctness of AUROC with tie, we computed the AUROC of the predictions on the AML dataset by \fw{} and compared it to the result obtained without privacy on the DREAM Challenge dataset. We obtained $AUROC = 0.693$ in both settings. To check the correctness of AUROC with no-tie of \fw{}, we randomly picked one of the samples in tie condition in DREAM Challange dataset and generated a subset of the samples with no tie. We obtained the same AUROC with no-tie version of AUROC of \fw{} as the non-private computation. We directly used the UCI dataset in AUROC with no-tie since it does not have any tie condition. The result, which is $AUROC = 0.927$, is the same for both private and non-private computation. Additionally, we verified that \fw{} computes the same AUPR as for the non-private computation for both the DREAM Challenge and the UCI dataset. AUPR scores are $AUPR=0.844$ and $AUPR=0.893$, respectively. These results indicate that \fw{} can privately compute the exact same AUC as one could obtain on the pooled test samples.

\begin{table*}[ht]
    \small
    \centering
    \renewcommand{\arraystretch}{1.2}
    \begin{tabularx}{\textwidth}{S{5}{1.47} S{10}{0.5} S{20}{1.85} S{20}{1.85} S{20}{1.85} S{20}{1.85} F}
    \toprule
     & & \multicolumn{4}{>{\columncolor{gray!20!white}}c}{Communication Costs (MB)} & \\
    \hhline{*1{>{\arrayrulecolor{gray!5!white}}-}*1{>{\arrayrulecolor{gray!10!white}}-}*4{>{\arrayrulecolor{black}}-}*1{>{\arrayrulecolor{gray!30!white}}-}}{\arrayrulecolor{black}}
    D $\times$ \nsamples & $\delta$ & $P_1$ & $P_2$ & Helper & Total &  Time (sec)\\ 

    \midrule{{\arrayrulecolor{black}}}
     
    $3 \times 64 $ & $1$ & $ 1.96 $ & $ 1.3 $ & $ 1.13 $ & $ 4.39 $ & $ 24.41 $ \\
    $3 \times 128 $ & $1$ & $ 6.61 $ & $ 4.14 $ & $ 3.97 $ & $ 14.72 $ & $ 48.05 $ \\
    $3 \times 256 $ & $1$ & $ 24.44 $ & $ 15.23 $ & $ 15.06 $ & $ 54.73 $ & $ 95.65 $ \\
    $3 \times 512 $ & $1$ & $ 93.23 $ & $ 58.44 $ & $ 58.26 $ & $ 209.93 $ & $ 191.55 $ \\
    $3 \times 1024 $ & $1$ & $ 359.67 $ & $ 226.62 $ & $ 226.41 $ & $ 812.7 $ & $ 355.32 $ \\
    \midrule
    $ 2 \times 1000$ & $1$ & $ 125.05 $ & $ 78.37 $ & $ 78.19 $ & $ 281.61 $ & $ 174.16 $ \\
    $ 4 \times 1000$ & $1$ & $ 726.44 $ & $ 458.81 $ & $ 458.57 $ & $ 1643.82 $ & $ 523.39 $ \\
    $ 8 \times 1000$ & $1$ & $ 3355.74 $ & $ 2125.91 $ & $ 2125.51 $ & $ 7607.16 $ & $ 1404.22 $ \\
    \midrule
    $8 \times 1000$ & $ 3 $ & $ 1692.85 $ & $ 1069.58 $ & $ 1069.25 $ & $ 3831.68 $ & $ 1194.08 $ \\
    $8 \times 1000$ & $ 5 $ & $ 1137.91 $ & $ 717.45 $ & $ 717.15 $ & $ 2572.51 $ & $ 1105.46 $ \\
    $8 \times 1000$ & $ 11 $ & $ 583.02 $ & $ 365.36 $ & $ 365.08 $ & $ 1313.46 $ & $ 1032.29 $ \\
    $8 \times 1000$ & $ 25 $ & $ 284.22 $ & $ 175.76 $ & $ 175.5 $ & $ 635.48 $ & $ 972.0 $ \\
    $8 \times 1000$ & $ 51 $ & $ 156.23 $ & $ 94.54 $ & $ 94.29 $ & $ 345.06 $ & $ 935.65 $ \\
    $8 \times 1000$ & $ 101 $ & $ 93.59 $ & $ 54.79 $ & $ 54.53 $ & $ 202.91 $ & $ 940.07 $ \\

    \midrule
    
    $8 \times UNB$ & $1$ & $ 130.48 $ & $ 81.82 $ & $ 81.6 $ & $ 293.9 $ & $ 379.99 $ \\

    \bottomrule
    \end{tabularx}
    \caption{Summary of the results of the experiments of AUPR computation with \fw{} on synthetic data. $D$ represents the number of data sources and $\nsamples$ represents the number of samples in one data source. $UNB$ represents the unbalanced sample distribution, which is $\{12,18,32,58,107,258,507,1008\}$.}
    \label{tab:syn_auprc_result_summary}
\end{table*}


\subsection{Scalability Analysis} 
We evaluated no-tie and with-tie versions of AUROC and AUPR of \fw{} with $\delta = 1$ on the settings in which the number of data sources is $3$ and the number of samples is $\nsamples \in \{64, 128, 256, 512, 1024\}$. The results showed that \fw{} scales almost quadratically in terms of both communication costs among all parties and the execution time of the computation. Figure \ref{fig:n_samples} displays the results. We also analyzed the performance of all computations of \fw{} on a varying number of data sources. We fixed $\delta = 1$ and the number of samples in each data sources to $1000$, and we experimented with $D$ data sources where $D \in \{2, 4, 8\}$. As Figure \ref{fig:n_parties} summarizes, \fw{} scales around quadratically to the number of data sources. We also analyzed the effect of $\delta \in \{3,5,11,25,51,101\}$ by fixing $D$ to $8$ and $\nsamples$ in each data source to $1000$. The execution time shown in Figure \ref{fig:delta} displays logarithmic decrease for increasing $\delta$. In all analyses, since the dominating factor is sorting, the execution times of the computations are close to each other. Additionally, our analysis showed that LAN is $12$ to $14$ times faster than WAN on average due to the high round trip time of WAN, which is approximately $10$ ms. However, even with such a scaling factor, \fw{} can be deployed in real life scenarios if the alternative is a more time-consuming approval process required for gathering all data in one place still protecting the privacy of data.

\section{Conclusion}
In this work, we presented an efficient solution based on a secure 3-party computation framework to compute AUC of the ROC and PR curves privately even when there exist ties in the PCVs. We benefited from the built-in building blocks of \cecilia{} and adapted the division operation of SecureNN to compute the exact AUC. \fw{} is secure against passive adversaries in the honest majority setting. We demonstrated that \fw{} can compute correctly and privately the exact AUC that one could obtain on the pooled plaintext test samples, and \fw{} scales quadratically to the number of both parties and samples. To the best of our knowledge, \fw{} is the first solution enabling the private and secure computation of the exact AUROC and AUPR.

\section*{Acknowledgement}
This study is supported by the DFG Cluster of Excellence “Machine Learning – New Perspectives for Science”, EXC 2064/1, project number 390727645 and the German Ministry of Research and Education (BMBF), project number 01ZZ2010.

\bibliographystyle{plainnat}  
\bibliography{references} 

\section*{Appendix}
\begin{landscape}
\setlength{\tabcolsep}{0.19cm}
\begin{table*}[ht]
    \small
    \centering
    \renewcommand{\arraystretch}{1.2}
    \begin{tabularx}{1.4\textwidth}{S{5}{1.07} S{10}{0.37} S{20}{2.45} S{20}{2.45} S{20}{2.45} S{20}{2.45} F}
    \toprule
     & & \multicolumn{4}{>{\columncolor{gray!20!white}}c}{Communication Costs (MB)} & \\
    \hhline{*1{>{\arrayrulecolor{gray!5!white}}-}*1{>{\arrayrulecolor{gray!10!white}}-}*4{>{\arrayrulecolor{black}}-}*1{>{\arrayrulecolor{gray!30!white}}-}}{\arrayrulecolor{black}}
    D $\times$ \nsamples & $\delta$ & $P_1$ & $P_2$ & Helper & Total &  Time (sec)\\ 
    \midrule
     $3 \times 64 $ & $1$ & $ 1.36 / 1.41 $ & $ 1.02 / 1.05 $ & $ 0.86 / 0.88 $ & $ 3.24 / 3.34 $ & $ 16.77 / 24.72 $ \\
    $3 \times 128 $ & $1$ & $ 5.49 / 5.63 $ & $ 3.64 / 3.71 $ & $ 3.48 / 3.55 $ & $ 12.61 / 12.89 $ & $ 33.52 / 48.96 $ \\
    $3 \times 256 $ & $1$ & $ 22.06 / 22.31 $ & $ 14.15 / 14.28 $ & $ 13.97 / 14.11 $ & $ 50.18 / 50.7 $ & $ 66.1 / 97.65 $ \\
    $3 \times 512 $ & $1$ & $ 87.8 / 87.32 $ & $ 55.85 / 55.49 $ & $ 55.66 / 55.31 $ & $ 199.31 / 198.12 $ & $ 131.69 / 193.25 $ \\
    $3 \times 1024 $ & $1$ & $ 349.64 / 353.11 $ & $ 221.96 / 224.07 $ & $ 221.75 / 223.86 $ & $ 793.35 / 801.04 $ & $ 233.27 / 356.31 $ \\
    
    \midrule
    
    $ 2 \times 1000$ & $1$ & $ 119.28 / 119.66 $ & $ 75.81 / 75.99 $ & $ 75.63 / 75.81 $ & $ 270.72 / 271.46 $ & $ 91.44 / 173.79 $ \\
    $ 4 \times 1000$ & $1$ & $ 714.87 / 715.42 $ & $ 453.69 / 454.06 $ & $ 453.45 / 453.82 $ & $ 1622.01 / 1623.3 $ & $ 358.37 / 523.06 $ \\
    $ 8 \times 1000$ & $1$ & $ 3333.01 / 3332.37 $ & $ 2115.66 / 2116.39 $ & $  2115.26 / 2115.99 $ & $ 7563.93 / 7564.75 $ & $ 1074.87 / 1404.6 $ \\
    
    \midrule
    $8 \times 1000$ & $ 3 $ & $ 1668.73 / 1671.24 $ & $ 1059.34 / 1060.06 $ & $ 1059.02 / 1059.73 $ & $ 3787.09 / 3791.03 $ & $ 818.98 / 1149.16 $ \\
    $8 \times 1000$ & $ 5 $ & $ 1114.3 / 1115.72 $ & $ 707.22 / 707.94 $ & $ 706.92 / 707.64 $ & $ 2528.44 / 2531.3 $ & $ 732.85 / 1061.33 $ \\
    $8 \times 1000$ & $ 11 $ & $ 559.28 / 561.35 $ & $ 355.1 / 355.83 $ & $ 354.83 / 355.56 $ & $ 1269.21 / 1272.74 $ & $ 646.66 / 1000.99 $ \\
    $8 \times 1000$ & $ 25 $ & $ 261.08 / 262.62 $ & $ 165.52 / 166.24 $ & $ 165.26 / 165.98 $ & $ 591.86 / 594.84 $ & $ 604.81 / 974.64 $ \\
    $8 \times 1000$ & $ 51 $ & $ 132.84 / 134.64 $ & $ 84.3 / 85.03 $ & $ 84.04 / 84.77 $ & $ 301.18 / 304.44 $ & $ 591.72 / 965.86 $ \\
    $8 \times 1000$ & $ 101 $ & $ 70.44 / 71.98 $ & $ 44.54 / 45.26 $ & $ 44.29 / 45.01 $ & $ 159.27 / 162.25 $ & $ 599.53 / 955.07 $ \\
    
    \midrule
    $8 \times UNB$ & $1$ & $ 120.51 / 120.56 $ & $ 76.6 / 76.57 $ & $ 76.38 / 76.35 $ & $ 273.49 / 273.48 $ & $ 297.59 / 379.71 $ \\
     
    \bottomrule
    \end{tabularx}
    \caption{Summary of the results of the experiments with \fw{} to compute AUROC with and without tie on synthetic data. The left side of ``/'' represents \textit{without-tie} results and the right side of it represents \textit{with-tie} results. $D$ represents the number of data sources and $\nsamples$ represents the number of samples in one data source. $UNB$ represents the unbalanced sample distribution, which is $\{12,18,32,58,107,258,507,1008\}$.}
    \label{tab:syn_auroc_result_summary}
\end{table*}
\end{landscape}

\end{document}

%% file: cecilia_notations.tex
\newcommand{\fw}{CECILIA}

\newcommand{\party}{P} 
\newcommand{\dec}{f} 
\newcommand{\numbit}{n} 

\newcommand{\ringA}{L} 

%% file: columns_and_rows.tex
\usepackage{colortbl}
\usepackage[usenames,dvipsnames,svgnames,table]{xcolor}

\newcolumntype{A}{>{\centering\columncolor{gray!5!white}\arraybackslash}X}
\newcolumntype{B}{>{\centering\columncolor{gray!10!white}\arraybackslash}X}
\newcolumntype{C}{>{\centering\columncolor{gray!20!white}\arraybackslash}p{1.7cm}}
\newcolumntype{D}{>{\centering\columncolor{gray!30!white}\arraybackslash}p{1.7cm}}
\newcolumntype{E}{>{\centering\columncolor{gray!20!white}\arraybackslash}X}

%% file: mathematical_notations.tex
\usepackage{mathtools}

\newcommand{\lf}{\langle} 
\newcommand{\rg}[1]{\rangle_{#1}} 

%% file: ppaurora_notations.tex
\newcommand{\pl}{L} 
\newcommand{\ring}{K} 
\newcommand{\nsamples}{M} 